\pdfoutput=1

\documentclass[11pt]{article}

\usepackage[preprint]{acl}

\usepackage{times}
\usepackage{latexsym}
\usepackage[english]{babel}
\usepackage[T1]{fontenc}

\usepackage[utf8]{inputenc}
\usepackage{multirow}
\usepackage{microtype}
\usepackage{amsfonts}
\usepackage{amsmath}
\usepackage{mathtools}
\usepackage{bbding}
\usepackage[most]{tcolorbox}
\usepackage{inconsolata}
\usepackage{amsthm}
\usepackage{float}
\usepackage{booktabs}
\usepackage{graphicx}
\usepackage{hyperref}

\newtheorem{definition}{Definition}[section]
\newtheorem*{remark}{Remark}
\newtheorem{theorem}{Theorem}[section]

\newtheorem{lemma}[theorem]{Lemma}
\title{SPPD: Self-training with Process Preference Learning Using Dynamic Value Margin}

\author{
  Hao Yi$^{1,2}$, 
 Qingyang Li$^{1}$\thanks{Corresponding author.},
 Yulan Hu$^{1,2}$,
 Fuzheng Zhang$^{1}$ ,
 Di Zhang$^{1}$,
 Yong Liu$^{2}$ \\
  $^1$ Kuaishou Technology, Beijing, China \\ 
  $^2$ Renmin University of China, Gaoling School of Artificial Intelligence, Beijing 
}

\begin{document}
\maketitle
\begin{abstract}
Recently, enhancing the numerical and logical reasoning capability of Large Language Models (LLMs) has emerged as a research hotspot. Existing methods face several limitations: inference-phase techniques (e.g., Chain of Thoughts) rely on prompt selection and the pretrained knowledge; sentence-level Supervised Fine-Tuning (SFT) and Direct Preference Optimization (DPO) struggle with step-wise mathematical correctness and depend on stronger models distillation or human annotations; while Reinforcement Learning (RL) approaches incur high GPU memory costs and unstable training. To address these, we propose \textbf{S}elf-training framework integrating  \textbf{P}rocess  \textbf{P}reference learning using \textbf{D}ynamic value margin (SPPD). SPPD leverages a process-based Markov Decision Process (MDP) and Bellman optimality equation to derive \textbf{dynamic value margin} on step-level preference optimization, which employs tree-based self-sampling on model responses \textbf{without any distillation} from other models. Furthermore, we theoretically prove that SPPD is \textbf{equivalent to on-policy policy gradient methods} under reward constraints. Experiments on 7B-scale models demonstrate superior performance across in-domain and out-domain mathematical benchmarks. We open-source our code at \href{https://anonymous.4open.science/r/SSDPO-D-DCDD}{https://anonymous.4open.science/r/SPPD-DCDD}.

\end{abstract}
\section{Introduction}
Recently, the O-series models \citep{o1} have achieved a significant leap in the mathematical reasoning capabilities of LLMs. Consequently, enhancing the numerical and logical reasoning capability of LLMs has emerged as a research hotspot \citep{self-debug,metamath,swe-bench,deepseek-math,tpo,step-dpo,deepseek-r1}. 

From now on, there are lots of methods to promote the model reasoning capability. During the \textbf{inference} phase, the most common and effective approach is to employ Chain of Thoughts (CoT) prompts, which can stimulate the model's inherent reasoning and thinking abilities \citep{cot}.  Similar methods include Tree of Thoughts (ToT) \citep{tot}, Best of N (BoN) \citep{qwen-prm,self-rewarding}, Monte Carlo Tree Search (MCTS) \citep{tsllm,rest-mcts}, and so on. However, these methods do not involve training policy models but rely on increasing computational volume during the inference phase, heavily depending on prompt selection and the pretrained knowledge embedded within the model. Moreover, \textbf{SFT} \citep{rest-mcts,tsllm} or \textbf{DPO} \citep{dpo,secret-q} based on human annotations or feedback from more advanced AI also serves as an effective way to enhance the model's reasoning capabilities. These methods leverage human-curated selections or stronger open-source and close-source models  to inject good reasoning paradigms, such as long-thought processes and reflection, into the model being trained. However, all these methods are at the sentence level, which does not align well with the requirement for correctness at every step in mathematical reasoning scenarios. Meanwhile, such methods are either constrained by time-consuming manual selection processes or require support from more powerful models, like STILL-2 \citep{slow-thinking} and Skywork-o1-open 
\citep{skywork-o1}. When the model to be trained is already the strongest reasoning model available, how can we further improve the model’s reasoning performance without any distillation? 
While \textbf{RL-based} methods like Proximal Policy Optimization (PPO) \citep{ppo}, Group Relative Policy Optimization (GRPO) \citep{deepseek-math,deepseek-r1}, Reinforcement Fine-Tuning (RFT) \citep{reft}, etc., can address the aforementioned issues. However, these methods are online approaches involving numerous time-consuming inference operations during training, requiring loading and training multiple models, imposing high demands on GPU memory and leading to highly unstable training processes. 

To solve above issues, we propose \textbf{S}elf-training with \textbf{P}rocess \textbf{P}reference learning using \textbf{D}ynamic value margin (SPPD). Unlike sentence-level SFT and DPO, we completely abandon the data distillation approach and propose optimizing at the step level by integrating dynamic value margin. Specifically, SPPD utilizes a process-based MDP and a process-based Bradley-Terry (BT) Model \citep{bt-model}. By leveraging the Bellman optimality equation \citep{bellman-equation} and the online RL objective modeled with MDP \citep{secret-q}, SPPD derives step-wise direct preference optimization using \textbf{dynamic value margin}. Additionally, SPPD \textbf{does not rely on any stronger models for data distillation}. Instead, it employs a tree search approach, which utilizes step-level trajectory sampling solely on the model's own response and logits score. To ensure smoother and more effective training of SPPD, we introduce an SFT and DPO strategy based on PRM rejection sampling, progressively enhancing the model's reasoning capabilities from coarse-grained sentence-level optimization to fine-grained step-level refinement. Finally, we theoretically prove that under specific reward constraints, our method \textbf{is equivalent to on-policy policy gradient method}.

The experimental results demonstrate that SPPD achieves widespread and significant improvements across different model architectures of 7B size and various in-domain and out-domain mathematical test datasets. It surpasses most existing open-source models of the same size and some closed-source models, demonstrating the effectiveness and robustness of SPPD. Our contribution are summarized as follows: 1) We utilize the Bellman optimality equation and the online RL objective modeled with MDP to achieve SPPD and iteratively improve the reasoning capability. 2)  We design a step-level tree self-sampling scheme without any distillation from stronger model.  3) We theoretically prove that our method is equivalent to on-policy policy gradient optimization.

\section{Related Work}
\textbf{Enhance Reasoning Capability of LLMs.}
Recently, a substantial body of research focuses on enhancing the reasoning capabilities of LLMs. These methodologies are primarily divided into two categories: the inference phase and the Post-Training phase. During the inference phase, early studies concentrate on stimulating the model's inherent reasoning abilities by modifying prompts \citep{cot,tot}. Subsequent research leverages the consistency of multiple inferences by the model \citep{self-rewarding,self-consistency} or integrates tree search strategies \citep{tsllm,rest-mcts} to guide the model towards more accurate decoding processes. However, these approaches do not involve training and heavily rely on the model's intrinsic reasoning capabilities. In the Post-Training phase, SFT \cite{tsllm} and DPO \cite{dpo,secret-q} emerge as primary enhancement techniques. These methods depend on human-curated selection of high-quality reasoning trajectories or distillation of responses from stronger models \cite{slow-thinking} to improve the reasoning performance of smaller or weaker models. Nevertheless, these approaches are time-consuming and unsustainable. RL paradigms, exemplified by PPO \citep{ppo}, GRPO \citep{deepseek-r1,deepseek-math}, and ReFT \citep{reft}, effectively address the aforementioned issues but introduce significant GPU memory consumption and training instability challenges. 
\\\textbf{Step-Level Direct Preference Optimization.}
In order to optimize and improve the model's reasoning capability from the step level, CPO \citep{cpo} aligns each step of the CoT reasoning paths with those of ToT using the inherent preference information in the tree-search process, but it control LLMs to generate the thought data by prompt, which may influent the model generation quality. Step-DPO \citep{step-dpo} treats individual reasoning steps as units for preference optimization. However, it utilizes the GPT4 to evaluate the correctness of step, which could bring introduced bias and is expensive.  TPO \citep{tpo} claims that the policy can potentially learn more effectively from a ranked preference list of responses given the prompt and utilizes adaptive step reward to adjust the reward values of each step in the trajectory.  However, it introduce a stronger form of ``catastrophic forgetting'' and imbalanced distribution of the preference tree reward values.

\section{Preliminaries}\label{sec:step-dpo-mdp}
In this section, we first define the step-level MDP in  natural language process. Subsequently, based on the step-level MDP, we modify the original RLHF objective and provide the optimal (fixed-point) solution to maximum casual entropy problem.

\textbf{Step-Level MDP in LLMs.}
We describe the step-level MDP in natural language process. The step-level MDP is defined as the following quintuple: $\mathcal{M}=(\mathcal{A},\mathcal{S},f,r,\rho_0)$, where $\mathcal{A}$ represents the set of action spaces, consisting of a reasoning step $a_t$; $\mathcal{S}$ represents the set of states, which in natural language denotes the sequence of the problem and the current reasoning step $s_t=s_0|a_1|a_2|...|a_t$, where | denotes the string concatenation operation and $s_0$ is the problem. It is noteworthy that the selection of $a_t$ depends on the current state. $f:\mathcal{S}\times \mathcal{A}\rightarrow\mathcal{S}$ represents the state transition function, indicating the transition from the current state to the next state after performing a certain action. Specifically, $f(s,a) = s|a$.  $r:\mathcal{S}\times\mathcal{A}\rightarrow\mathbb{R}$ is the reward function, representing the immediate reward obtained after performing a certain action in the current state. $\rho_0$ represents the distribution of the problems.

\textbf{RLHF objective with the Step-Level MDP.} In the original RLHF objective \citep{rlhf-objective}, the rewards obtained from trajectories are modeled as a bandit problem \citep{rlhf_bandit}. However, such sparse rewards are not suitable for policy learning in models, especially in mathematical reasoning tasks \citep{reward_sparse_1,reward_sparse_2}. Based on the step-level MDP, we modify the RLHF objective as follows \citep{secret-q}:
\begin{align}
\max_{\pi_\theta} \mathbb{E}_{a_t \sim \pi_\theta(\cdot \mid \mathbf{s}_t)} [ &\sum_{t=0}^T (r(\mathbf{s}_t, \mathbf{a}_t) + \underbrace{\beta \log \pi_{\text{ref}}(\mathbf{a}_t \mid \mathbf{s}_t)}_{\text{KL penalty}} ) \nonumber \\
&+\beta \mathcal{H}(\pi_\theta) \mid \mathbf{s}_0 \sim \rho(\mathbf{s}_0) ],
\label{equ:rlhf-step-mdp}
\end{align}
where   $\pi_\theta$ represents the large language policy model with learnable parameters, $\pi_{ref}$ represents reference model and $\beta$ is used to control the policy model not to deviate too far from the reference model, \(\mathcal{H}(\pi_\theta)\) is the entropy of  $\pi_\theta$. This optimization problem is known as the \textbf{Maximum Causal Entropy}. \citet{fix_solutions} have proven that Equation (\ref{equ:rlhf-step-mdp}) has a fixed-point solution \(\pi^*\), defined as follows:
\begin{align}
    \pi^*(a_t \mid s_t) = \pi_{\text{ref}}(a_t|s_t) e^{(Q^*(s_t, a_t) - V^*(s_t)) / \beta},
    \label{equ:fix_solution}
\end{align}
where \(V^*(s_t)\) represents the partition function of the \(\pi^*\) distribution, used to normalize the probability distribution, and \(Q^*(s_t, a_t)\) denotes the expected sum of future immediate rewards starting from the state-action pair \((s_t, a_t)\) under the policy \(\pi^*\).
\section{Method}
\begin{figure}
    \centering
    \includegraphics[width=1.0\linewidth]{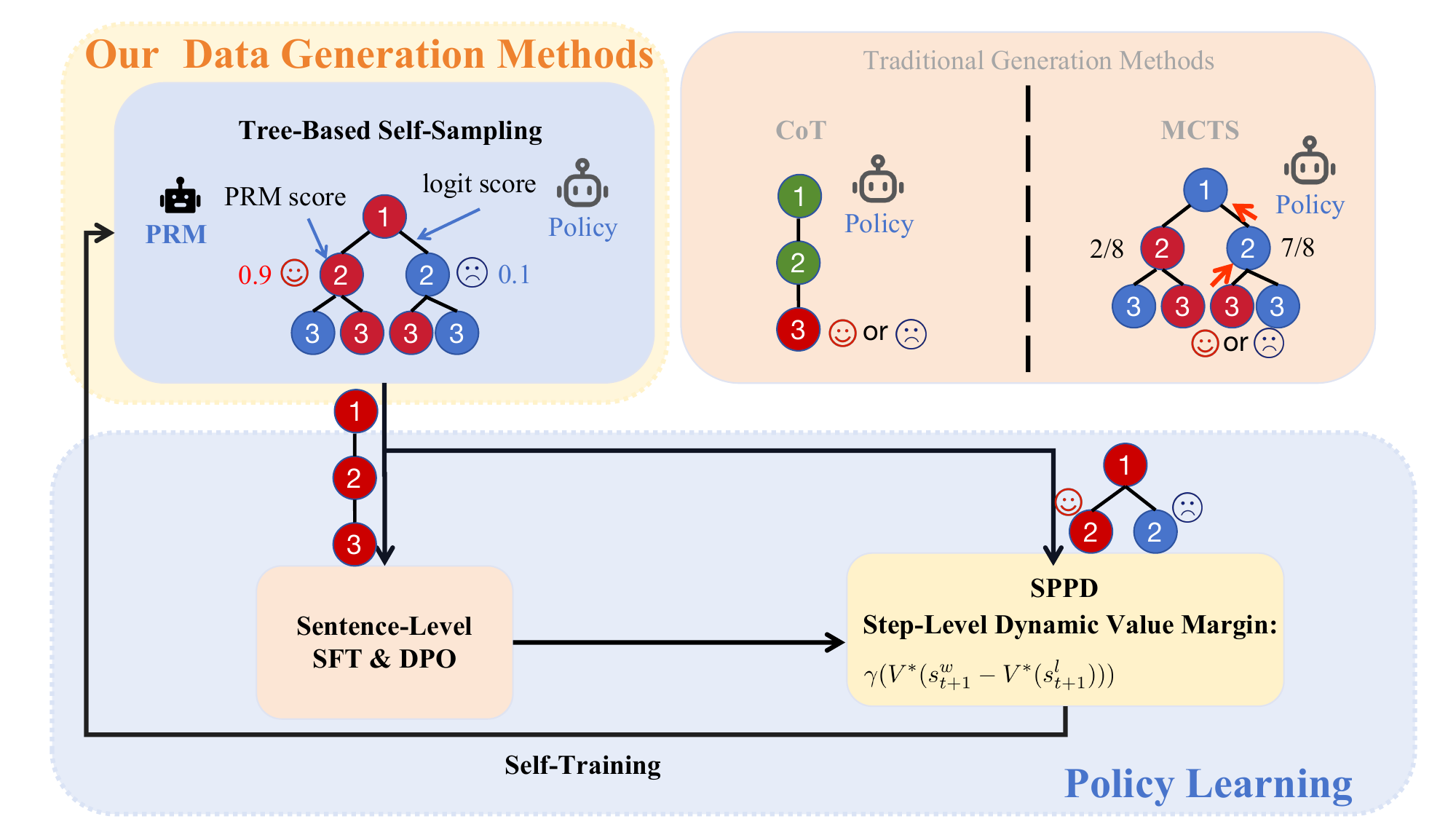}
    \caption{The framework of SPPD: unlike CoT and MCTS, Tree-Based Self-Sampling generates step trajectories with common prefixes and significantly preserves the output distribution of the policy. The former provides step preference signals for SPPD, while the latter theoretically ensures consistency with on-policy gradient methods, thereby enabling self-enhancement of the model's reasoning capabilities.}
    \label{fig:framework}
\end{figure}
In this section, we first propose a process preference learning scheme using dynamic value margins based on the step MDP and BT-model, and then refine this preference learning scheme using the reward equivalence. Additionally, we introduce a tree-based self-sampling method designed to generate step trajectories with common prefix. Finally, we introduce sentence-level SFT and DPO using PRM, aiming to make the model training smoother and more effective. 
\subsection{Process Preference Learning with Dynamic Value Margin}\label{sec:step-dpo}\label{sec:method-prm-step-dpo}
First, we derive the process preference learning with dynamic value margin starting from the optimal Bellman equation and revisit the traditional step DPO \citep{step-dpo} from a different perspective. 

\begin{lemma}[Optimal Step Reward Function]\label{lemma:reward-function}
    Under the step MDP definition in Section \ref{sec:step-dpo-mdp} and fix solution for the maximum casual entropy problem (Equation (\ref{equ:fix_solution})), the optimal step reward function can be calculate as follow:
    \begin{align}
    r(s_t, a_t) = \underbrace{\beta \log \frac{\pi^*(a_t | s_t)}{\pi_{\text{ref}}(a_t | s_t)}}_{\text{Implicit Reward}} + \underbrace{V^*(s_{t+1}) - V^*(s_t)}_{\text{Value Gain}}.
    \label{equ:reward}
\end{align}
\end{lemma}

Proof of Lemma (\ref{lemma:reward-function}) is shown in Appendix \ref{app:prove_for_reward_function}. Equation (\ref{equ:reward}) demonstrates that the immediate reward in the MDP consists of the model's \textbf{implicit reward} and the \textbf{value gain} of the optimal value function. Assuming we have the following step-level preference pairs \((s_t, a^w_{t+1}, a^l_{t+1})\), based on the step-level BT-model, we have the optimal preference distribution:
\begin{align*}
    p^*(a^w_{t+1} \succ a^l_{t+1}) = \sigma\left(r(s_t, a^w_{t+1}) - r(s_t, a^l_{t+1})\right).
\end{align*}
Here, \(\sigma(x) = 1 / (1 + e^{-x})\) is the sigmoid function.  Finally, we give the step DPO  loss using  dynamic value margin.
\begin{theorem}[Step DPO  Loss Using  Dynamic Value Margin.]\label{theorem:step-dpo-loss}
   If we aim to minimize the Kullback-Leibler(KL) divergence between the step-level preference distribution \(p_{\text{data}}\) in \(\mathcal{D}_{\text{step}}\) and the model's current preference distribution \(p_{\theta}\)  under the sampling of \(\pi_{ref}\), we can obtain the following loss function:
\begin{align}
    \mathcal{L}_{\text{step-dpo}} &=  -\mathbb{E}_{a^w_{t+1}, a^l_{t+1} \sim \pi_{\text{ref}}(\cdot | s_t)} [ \nonumber \\
   & \log \sigma (\beta h_{\theta}(a^w_{t+1}, a^l_{t+1}) \nonumber \\
    & - (V^*(s^w_{t+1}) - V^*(s^l_{t+1}))) ],
    \label{equ:step-dpo-loss}
\end{align}
where $h_{\theta}(a^w_{t+1}, a^l_{t+1}) =\log \frac{\pi_{\theta}(a_t^w | s_t)}{\pi_{\text{ref}}(a_t^w | s_t)} - \log \frac{\pi_{\theta}(a_t^l | s_t)}{\pi_{\text{ref}}(a_t^l | s_t)}$. 
\end{theorem} The proof is shown in Appendix \ref{app:prove_steo_dpo_loss}. In traditional step DPO \citep{step-dpo}, the value function prediction at each step is defined as 0. However, we argue that the \textbf{value gain} in the immediate reward (Equation (\ref{equ:reward})), or equivalently, the term $V^*(s^w_{t+1}) - V^*(s^l_{t+1})$ in Equation (\ref{equ:step-dpo-loss}), considers the difference in the optimal value function predictions for the preferred states. This manifests in the step DPO loss as a dynamic value margin that varies depending on the preferred states \(s^w_{t+1}\) and \(s^l_{t+1}\), rather than treating all states uniformly. In practice, we use a PRM score to approximate the optimal value function. In Section \ref{sec:theoretical}, we will provide more profound theoretical insights and conclusions.

\textbf{Reward Equivalence.} To make the optimization process more controllable, we revise Equation (\ref{equ:reward}) by introducing the concept of reward equivalence.
\begin{lemma}Reward Equivalence \citep{secret-q}]
    Two reward functions \( r \) and \( r' \) are equivalent if and only if there exists a potential function \( \Phi: \mathcal{S} \rightarrow \mathbb{R} \) that satisfies the following equation:
    \begin{align*}
         r(s_t, a_t) = r'(s_t, a_t) + \Phi(f(s_t, a_t)) - \Phi(s_t). 
    \end{align*}
\end{lemma}
In Equation (\ref{equ:reward}), the potential function is our optimal value function, i.e., \( \Phi(s) = V^*(s) \). At the same time, it is easy to see that when we scale this potential function, \( \Phi'(s) = \gamma \Phi(s) \), \( \Phi' \) still satisfies the definition of  potential function. Therefore, we can modify Equation (\ref{equ:reward}) to obtain an equivalent reward expression:
\begin{align*}
     r^{\gamma}(s_t, a_t) = r(s_t, a_t) + \gamma\Phi(f(s_t, a_t)) - \gamma\Phi(s_t). 
\end{align*} Repeating the derivation in Section  \ref{sec:method-prm-step-dpo} , we modify the final loss as follows:
\begin{align}
    \mathcal{L}^{\gamma}_{\text{step-dpo}}
    &=  -\mathbb{E}_{a^w_{t+1}, a^l_{t+1} \sim \pi_{\text{ref}}(\cdot | s_t)} [ \nonumber \\
   & \log \sigma (\beta h_{\theta}(a^w_{t+1}, a^l_{t+1}) \nonumber \\
    & - \gamma(V^*(s^w_{t+1}) - V^*(s^l_{t+1}))) ].
    \label{equ:step-dpo-gamma}
\end{align}
\begin{remark}
    Although the concept of reward equivalence in \citet{secret-q} implies that the optimal preference model belongs to the same equivalence class, including the original step-DPO when \(\gamma = 0\), the introduction of \(\gamma\) makes the optimization process more controllable due to its influence on optimization. This has been verified in Section \ref{sec:exp-gamma}.
\end{remark}

\subsection{Tree-Based Self-Sampling on LLMs}\label{sec:step-data}

 Traditional reasoning algorithms (token-level decoding) is almost impossible to guarantee the generation of reasoning trajectories with identical prefixes. To address this issue, this paper adopts a tree-structured reasoning approach, as illustrated in Figure \ref{fig:framework}. Specifically, the process is divided into four steps: ``Selection, Expansion, Collection and Scoring''.   During the selection process, at the current state \(s_t\), we record the average log probability score for each child node \(a_t\), defined as:
\begin{align*}
    s(a_t | s_t) = \frac{1}{|a_t|} \sum_{i=0}^{|a_t|} \log \pi_{\text{infer}}(a_{t,i} | s_t | a_{t,<i}),
\end{align*}
where \(|a_t|\) represents the token length of the current step, \(a_{t,<i}\) denotes the first \(i-1\) tokens of \(a_t\), and \(\pi_{\text{infer}}\) represents the probability distribution output of the inference model (policy in RL). In practice, we set $\pi_{\text{infer}} = \pi_{\text{ref}}$. Furthermore, we normalize the score distribution of all child nodes and perform sampling to select  child nodes. Each selection starts from the root node and proceeds until reaching a leaf node that contains the final answer. If a node is not a terminal node and has no child nodes, we expand the node to obtain \(C\) possible reasoning steps. After performing the above steps \(K\) times, we traverse the expanded prefix tree and collect all answers that contain complete reasoning paths. Finally, we invoke the PRM  to score each step of the reasoning trajectory, resulting in the final step-level dataset:
\begin{align*}
    \mathcal{D}_{\text{step}} = \{ &( s_0^{(i)}, s_t^{(i,j)}, v_t^{(i,j)} ) \nonumber \\
    &\mid i \in [N], j \in [K], t \in |\tau^{(i,j)}| \},
\end{align*}
where $N$ is the number of the problems, \(v_t^{(i,j)}\) represents the PRM score of the state \(s_t^{(i,j)}\) in the \(j\)-th prefix sequence of the problem \(s_0^{(i)}\).

\subsection{PRM-Enhanced SFT \& DPO}

To make the model's learning process smoother, we introduce the concept of curriculum learning, initially allowing the model to learn strategies at the sentence-level. This step leverages the signal responses from the PRM on sampled trajectories to perform rejection sampling, and employs both supervised learning and preference learning to continuously improve the model's reasoning capabilities. Specifically, we define the following positive and negative sample trajectories:
\begin{align*}
    \tau_{+}^{(i)} = \max_{j\in [K]} \min_{v_t^{(i,j)}} \mathcal{D}_{\text{step}}^+, \\ 
    \tau_{-}^{(i)} = \min_{j\in [K]} \min_{v_t^{(i,j)}} \mathcal{D}_{\text{step}}^-.
\end{align*}
Here, \(\mathcal{D}_{\text{step}}^+\) and \(\mathcal{D}_{\text{step}}^-\) represent complete trajectories with correct and incorrect final answers, respectively. During the SFT phase, we minimize the next token prediction loss on \(\tau_{+}^{(i)}\). In the DPO phase, we select positive samples from \(\{\tau_{+}^{(i)}\}_{i=1}^N\) and negative samples from \(\{\tau_{-}^{(i)}\}_{i=1}^N\), thereby constructing preference samples for sentence-level DPO. We emphasize that SFT and DPO optimize the model's reasoning capabilities at a coarse-grained level, aiming to warm up the model's reasoning abilities and lay the foundation for subsequent step-level preference learning.

\section{Theoretical Analysis}\label{sec:theoretical}
In this Section, we prove that the equivalence between offline step DPO and online policy gradient under the specific reward definition. 
\begin{definition}[Preference decoding model \(\pi^p_\theta\) induced by \(\pi_\theta\)]
    Assume that when \(s = s_t\), the possible action space $\mathcal{A}_t = \{a_{t+1}^w,a_{t+1}^l\}$ . We define \(\pi^p_\theta\) as the following parameterized distribution:
    \begin{align*}
    \pi^p_\theta(a^w_{t+1} | s_t) = \sigma(r^w_{\theta,t} - r^l_{\theta,t}),
    \end{align*}
    where,\begin{align*}
            r^w_{\theta,t} = \beta \log \frac{\pi_\theta(a^w_{t+1} | s_t)}{\pi_{\text{ref}}(a^w_{t+1} | s_t)} - V^*(s^w_{t+1}) + V^*(s_t),\\
    r^l_{\theta,t} = \beta \log \frac{\pi_\theta(a^l_{t+1} | s_t)}{\pi_{\text{ref}}(a^l_{t+1} | s_t)} - V^*(s^l_{t+1}) + V^*(s_t).
    \end{align*}
\end{definition}

\begin{remark}
The preference decoding model \(\pi^p_\theta\) can be viewed as performing sampling on a binary prefix tree based on preference probabilities. This model relies on the probability outputs of the standard language model \(\pi_\theta\).
\end{remark}

\begin{lemma}[Online Policy Gradient on \(\pi^p_\theta\) \citep{rpg} ]
For any MDP, the expected long-term reward on \(\pi^p_\theta\) is given by \(J(\theta) = \sum_{\tau} \pi^p_{\theta}(\tau) r(\tau)\), where \(r(\tau)\) represents the long-term reward of trajectory \(\tau\). The policy gradient of this expected long-term reward on  \(\pi^p_\theta\) is:
\begin{align}
    \nabla_{\theta} J(\theta) = \mathbb{E}_{\tau \sim \pi^p_{\theta}} \left[ r(\tau) \sum_{t=0}^{T-1} \nabla_{\theta} \log \pi^p_\theta(a^w_{t+1} | s_t) \right].
    \label{equ:online-gd}
\end{align}

\end{lemma}


\begin{theorem}[Equivalence Between Offline Step DPO and Online Policy Gradient]\label{theorem:equivalence}
 If we define the reward in Equation (\ref{equ:online-gd}) as \(r(\tau) = \prod_{i=1}^T \frac{\pi_{\text{ref}}(a_t | s_t)}{\pi^p_\theta(a_t | s_t)}\), and define the \textbf{Offline every-step preference loss} as:
\begin{align}
    \mathcal{L}_{\text{every-step}} = \nonumber \\ 
    &\mathbb{E}_{\tau \sim \pi^p_{\text{ref}}} \left[- \sum_{t=0}^{T-1}  \log \pi^p_\theta(a_{t+1}^w | s_t) \right],
    \label{equ:every-step}
\end{align}
then the following equivalence holds:
\begin{align*}
    \nabla_{\theta} J(\theta) = -\nabla_{\theta} \mathcal{L}_{\text{every-step}}.
\end{align*}
\label{the:equ-off-on}
\end{theorem}
The proof is shown in Appendix \ref{app:theorem:equivalence}. 
\begin{remark}
    It is easy to see that \(\mathcal{L}_{\text{every-step}}\) (Equation (\ref{equ:every-step})) can be considered as the equivalent expression of \(\mathcal{L}_{\text{step-dpo}}\) (Equation (\ref{equ:step-dpo-loss})) when the sampling tree branches at \(C=2\) and preference sampling is performed for every action at each step. 
    Theorem \ref{the:equ-off-on}  demonstrates that,\textbf{ under the specific definition of the reward, optimizing the gradient of the offline preference loss is equivalent to the policy gradient of the preference decoding model in the online setting.} Additionally, for the definition of the reward \(r(\tau)\), when the reward is large, it indicates that the trajectory probability output \(\pi^p_\theta(\tau)\) of the preference decoding model is relatively small. To reduce the overall loss, the optimization process will focus more on the loss of this particular trajectory at this step.
\end{remark}


\begin{table*}[!hbt]
\centering
\begin{tabular}{c|c|cc|cc}
\toprule
Model     & Size  & Open & General     & MATH500 & GSM8k \\ \midrule
Claude-3-Opus* & - & \XSolidBrush & \CheckmarkBold & 60.1 & 95.0 \\
GPT4-1106 \citep{gpt4}* & - & \XSolidBrush & \CheckmarkBold & 64.3 & 91.4\\
GPT4o-0513* & - & \XSolidBrush & \CheckmarkBold & 76.6 & 95.8 \\
o1 \citep{o1}* & - & \XSolidBrush & \CheckmarkBold & 94.8 & - 
\\ \hline
Qwen2-7B-Instruct-Step-DPO \citep{step-dpo} & 7B &\CheckmarkBold &\XSolidBrush &  55.0 & 85.4 \\
DeepSeek-MATH-7B-Instruct \citep{deepseek-math} &7B & \CheckmarkBold &\XSolidBrush & 44.4 & 80.9 \\
OpenMath2-Llama3.1-8B \citep{openmath} & 8B & \CheckmarkBold & \XSolidBrush & 65.4  & 90.1 \\
Llama3.1-8B-Instruct \citep{Llama3p1} & 8B & \CheckmarkBold &\CheckmarkBold &47.0 & 82.6
\\
Qwen2.5-7B-Instruct \citep{qwen2.5}& 7B &\CheckmarkBold &\CheckmarkBold & 72.8 & 89.3
\\\hline
Qwen2.5-7B-Base  & 7B & \CheckmarkBold &  \CheckmarkBold & 60.0    & 82.3  \\ 
+SFT-PRM         & 7B &   \CheckmarkBold &  \XSolidBrush & 64.4    & 88.1  \\
+SFT-PRM \& DPO-PRM  & 7B &   \CheckmarkBold & \XSolidBrush  & 68.2    & 89.3  \\
+SPPD          & 7B &  \CheckmarkBold  &    \XSolidBrush    & \begin{tabular}[c]{@{}c@{}}\textbf{71.0}\\ \textcolor{blue}{+2.8\%}\end{tabular}          & \begin{tabular}[c]{@{}c@{}}\textbf{89.8}\\\textcolor{blue}{+0.5\%}\end{tabular}   \\
+SPPD+MAJ@64     & 7B &  \CheckmarkBold  &   \XSolidBrush  & 76.4    & 93.2  \\
+SPPD+ORM\_MAX@64  & 7B &   \CheckmarkBold & \XSolidBrush   & 74.0    & \textbf{94.9}  \\
+SPPD+ORM\_VOTE@64  & 7B &  \CheckmarkBold  & \XSolidBrush  & \textbf{79.0}    & 94.7 \\ 
+SPPD-Stage2 & 7B & \CheckmarkBold  & \XSolidBrush  &  \begin{tabular}[c]{@{}c@{}}\textbf{72.2}\\\textcolor{blue}{+4.0\%}\end{tabular} & \begin{tabular}[c]{@{}c@{}}\textbf{90.3}\\\textcolor{blue}{+1.0\%}\end{tabular} \\
+SPPD-Stage2+MAJ@64     & 7B &  \CheckmarkBold  &   \XSolidBrush  & 78.6    & 93.6  \\
+SPPD-Stage2+ORM\_MAX@64  & 7B &   \CheckmarkBold & \XSolidBrush   & 78.0    & \textbf{95.0}  \\
+SPPD-Stage2+ORM\_VOTE@64  & 7B &  \CheckmarkBold  & \XSolidBrush  & \begin{tabular}[c]{@{}c@{}}\textbf{80.4}\\\textcolor{blue}{+12.2\%}\end{tabular}
  & \begin{tabular}[c]{@{}c@{}}94.6\\\textcolor{blue}{+5.3\%}\end{tabular} \\ 
\bottomrule
\end{tabular}
\caption{Main Results. * denotes we use officially reported results.}
\label{tab:main_result}
\end{table*}

\section{Experiments}
\subsection{Setup}
\textbf{Datasets.} For the training prompt data, we sample a total of 10k prompts from the training datasets of GSM8k \citep{gsm8k} and MATH \citep{MATH}, with GSM8K and MATH accounting for 40\% and 60\% respectively. We use Qwen2.5-7B-Base \citep{qwen2.5} and Llama3.1-8B-Instruct \citep{Llama3p1} as the base models, and employ Skywork-o1-Open-PRM-Qwen-2.5-7B \citep{skywork-prm} as PRM to generate $\mathcal{D}_{\text{step}}$ using the step data generation method mentioned in Section \ref{sec:step-data}. For more information regarding the data format and PRM, please refer to the Appendix \ref{sec:data-example} \& \ref{prm-distribution}.\\
\textbf{Evaluation.} The maximum generation length for inference is set at 2048. The test set includes in-domain subsets such as GSM8k and MATH500, as well as out-domain subsets like Gaokao2023 \citep{gaokao}, OCW Course (OCW) \citep{ocw}, and the OlympiadBench \citep{Olympiadbench} test subset OE-TO-MATH-COMP. 
The testing methods comprise: 
    1) \textbf{Greedy-CoT}: Test results based on greedy decoding and CoT prompt pass@1.
    2)  \textbf{MAJ@N}: Repeat inference $N$ times based on the CoT prompt, and select the most frequently occurring answer as the final answer.
    3) \textbf{ORM\_VOTE@N}: Repeat inference $N$ times based on the CoT prompt, use Skywork-o1-Open-PRM-Qwen-2.5-7B as the ORM for scoring, aggregate scores for identical answers, and choose the answer with the highest score.
    4) \textbf{ORM\_MAX@N}: Omit the step of aggregating scores for identical answers in  \textbf{ORM\_VOTE@N} and directly select the answer with the highest score. 
More evaluation methods refer to Appendix \ref{app:evaluation}.\\
\textbf{Implementation.} During the data generation phase, we perform tree sampling for each question with a count of $K=64$, and each node branches into $C=2$. When selecting step-level preference pairs, to mitigate the impact of PRM scoring noise, we only use action preference pairs with a scoring difference exceeding 0.5 for training (PRM scores range between $0$ and $1$). In the SFT phase, we use the Adam optimizer with a learning rate of 5e-6, while in the DPO and step-DPO phases, we employ the SGD optimizer with a learning rate of 1e-5, both utilizing the cosine method for learning rate decay. The $\beta$ for both DPO and step DPO is set to 0.1. The $\gamma$ for step DPO is chosen from \{0.1,0.5,1.0,2.0,5.0\}. All experiments are conducted on 8 Nvidia 80GB H800 GPUs.

\subsection{Main Result}
\textbf{Compared to the base model}: Our approach achieves significant improvements without utilizing any stronger model's responses for distillation shown in Table \ref{tab:main_result}. Specifically, using SFT-PRM, we observe enhancements of 4.4\% and 5.8\% on the in-domain evaluation datasets MATH and GSM8k, respectively. With DPO-PRM, the improvements are 3.8\% and 1.2\%, respectively, on these same datasets. Building on this foundation, we further enhances the model's reasoning capabilities using SPPD, achieving additional improvements of 2.8\% and 0.5\% on the two evaluation datasets. The gains from SPPD stem from leveraging PRM signals, transitioning from coarse-grained optimization at the sentence level to fine-grained dynamic optimization at the step level. Additionally, during the inference phase, increasing computational load and employing the \textbf{ORM\_{VOTE}} aggregation strategy further demonstrates the model's peak reasoning capabilities, achieving accuracies of 79\% and 94.7\% on MATH and GSM8k, respectively, outperforming current models of similar size.
\\\textbf{Continued gains in the second stage}: In the first stage, the training data generated by the base model has been fully utilized. Following the principles of offline RL, we update the policy model’s sampling trajectories, using the best model trained in the first stage as our new policy model to repeat our training process. This resulted in the SPPD-Stage2 model. Compared to SPPD, SPPD-Stage2 achieves further improvements of 1.2\% and 0.5\% on MATH and GSM8k, respectively. These results highlight the effectiveness of updating the policy model and demonstrate the robustness of the SPPD.
\subsection{Ablation Study}\label{sec:exp-gamma}
\textbf{Different Base Model.}
We evaluate the effectiveness of the SPPD method on different base models, specifically Llama3.1-8B-Instruct and Qwen2.5-7B-Instruct. Given that Instruct models undergo sufficient optimization at the sentence level, we do not perform PRM-SFT and PRM-DPO training on these models. Instead, we directly utilize the trajectories from the Instruct models for dynamic value margin step DPO training. The results appear in Table \ref{tab:ablation-diff-model}. The findings indicate that on the Llama3.1-8B-Instruct model, the SPPD method achieves improvements of 4.6\% and 3.6\% on the MATH and GSM8k evaluation datasets, respectively. On the Qwen2.5-7B-Instruct model, the SPPD method improves performance by 2.2\% and 0.8\%, respectively. These experimental results demonstrate that the SPPD method performs well across different base models, showcasing its robustness with respect to the choice of base model. 
\begin{table}[h]
\resizebox{0.5\textwidth}{!}{
\begin{tabular}{ccc}
\toprule
Model    & MATH500  & GSM8K \\ \midrule
Llama3.1-8B-Instruct & 46.6  & 81.2  \\ \hline
+SPPD               & \begin{tabular}[c]{@{}c@{}}51.2\\ \textcolor{blue}{+4.6\%}\end{tabular} & \begin{tabular}[c]{@{}c@{}}84.8\\ \textcolor{blue}{+3.6\%}\end{tabular} \\
+SPPD+MAJ@64 & 58.2 & 88.5      \\
+SPPD+ORM\_MAX@64   & \textbf{ 67.0} &\textbf{ 92.0}  \\
+SPPD+ORM\_VOTE@64  &  \begin{tabular}[c]{@{}c@{}}66.4\\ \textcolor{blue}{+19.8\%}\end{tabular} & \begin{tabular}[c]{@{}c@{}}90.7\\ \textcolor{blue}{+9.5\%}\end{tabular} \\  \hline
Qwen2.5-7B-Instruct &  72.8 &  89.3  \\ 
+SPPD               & \begin{tabular}[c]{@{}c@{}}75.0\\ \textcolor{blue}{+2.2\%}\end{tabular} & \begin{tabular}[c]{@{}c@{}}91.1\\ \textcolor{blue}{+0.8\%}\end{tabular} \\
+SPPD+MAJ@64 & 80.6 & 93.4      \\
+SPPD+ORM\_MAX@64   &  77.0 & \textbf{95.2}  \\
+SPPD+ORM\_VOTE@64  & \begin{tabular}[c]{@{}c@{}}\textbf{82.2}\\ \textcolor{blue}{+9.4\%}\end{tabular} & \begin{tabular}[c]{@{}c@{}}94.6\\ \textcolor{blue}{+5.3\%}\end{tabular} \\  \bottomrule                    
\end{tabular}
}
\caption{Result on Llama3.1-8B-Instruct and Qwen2.5-7B-Instruct.}
\label{tab:ablation-diff-model}
\end{table}
\\\textbf{Generalization on Out-Domain Distributions.}
To evaluate the generalization capabilities of SPPD on out-domain distributions, we select three out-domain evaluation datasets: GaoKao2023, OCW and OlympaidBench (using only the OlympaidBench-OE-TO-MATH-COMP portion). The results are presented in Table \ref{tab:ablation-out-domain}. The experiments show that using Qwen2.5-7B-Base as the base model, after applying SPPD, there are steady improvements across all three out-of-domain evaluation datasets. Specifically, improvements over the base model stand at 8.8\%, 13.7\%, and 5.6\%, respectively. Over PRM-DPO, the improvements reach 1.8\%, 4.8\%, and 2.4\%, respectively. Furthermore, the reasoning capabilities see further enhancement through the ORM\_{VOTE} aggregation strategy. 
\\ \textbf{Effectiveness of Dynamic Value Margin.}
In Section \ref{sec:step-dpo}, we model the dynamic value margin variation using MDP approach, deriving a step DPO method with dynamically changing margins from a mathematical perspective. To validate the effectiveness of this dynamic value margin approach, we use Qwen2.5-7B-Base and Llama3.1-8B-Instruct as base models, followed by PRM-SFT and PRM-DPO training. We then compare SPPD with both no-margin step DPO ($\gamma = 0$) and fixed-margin step DPO. The results are summarized in Table \ref{tab:ablation-fix-margin}. The findings reveal that fixed-margin step DPO outperforms no-margin step DPO, indicating that adjusting the margin benefits the learning process of step DPO. Meanwhile, Compared to fixed-margin step DPO, SPPD demonstrates superior performance. On the Qwen model, improvements on MATH and GSM8k are 0.9\% and 0.31\%, respectively, while on the Llama model, the improvements are 2.0\% and 1.3\%, respectively. This improvement stems from our consideration of the value model score differences between preference pairs during modeling, which dynamically adjusts the margin for preference learning based on signals from the value model. SPPD makes the step-level preference training more reliable and reduces the risk of overfitting.
\\\textbf{Impact of $\gamma$.}
\begin{figure}[!ht]
    \centering
    \includegraphics[width=1.0\linewidth]{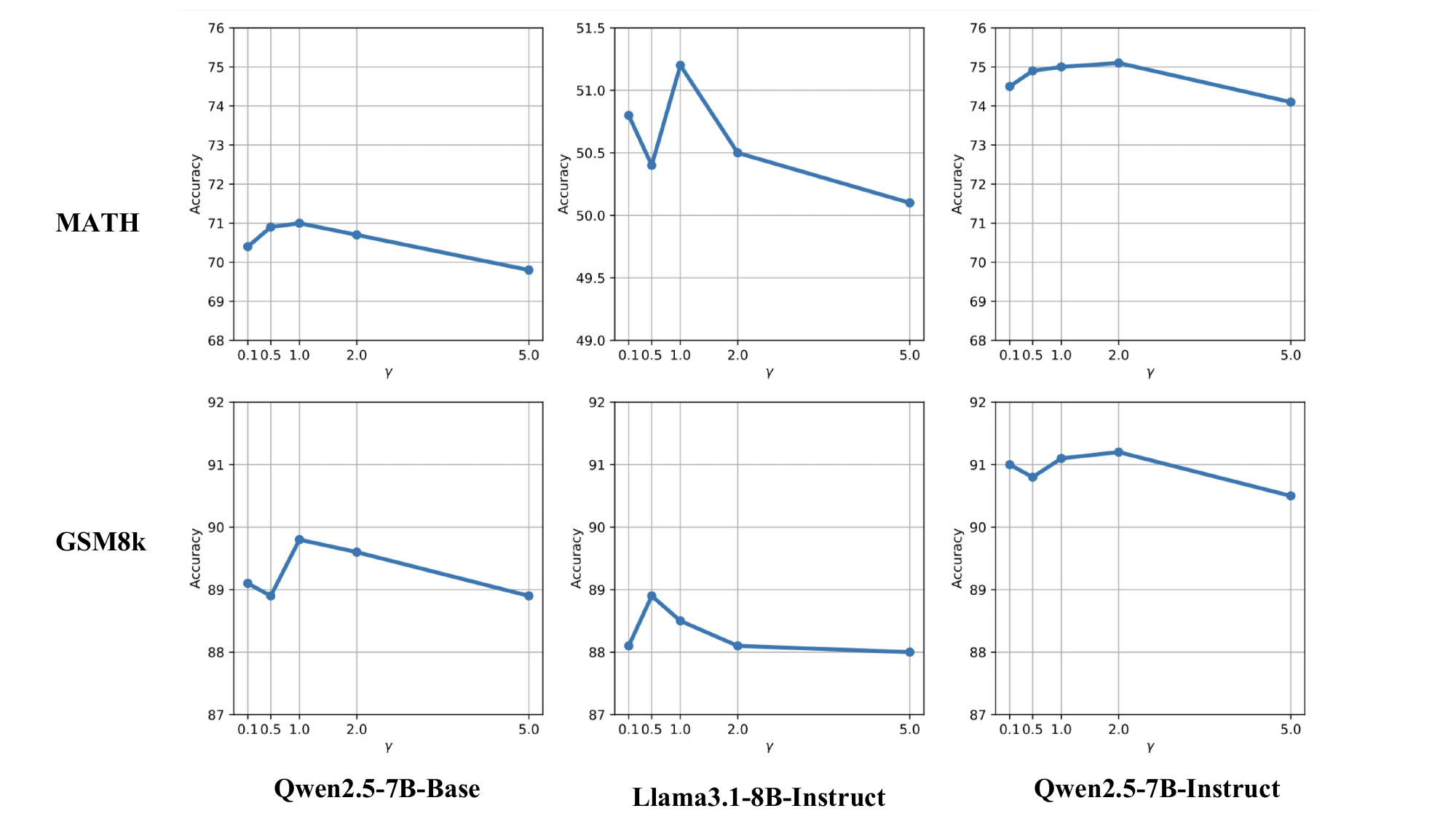}
    \caption{Impact of $\gamma$ in dynamic value margin.}
    \label{fig:gamma}
\end{figure}
To investigate the impact of the hyperparameter \(\gamma\) on the SPPD method as described in Formula \ref{equ:step-dpo-gamma}, we selecte three base models: Qwen2.5-7B-Base, Llama3.1-8B-Instruct, and Qwen2.5-7B-Instruct. We adjust \(\gamma\) within the set \(\{0.1, 0.5, 1.0, 2.0, 5.0\}\) and evaluated the performance of these models on the MATH and GSM8k datasets. The results are presented in Figure \ref{fig:gamma}. Our experimental findings indicate that selecting an appropriate \(\gamma\) is beneficial for the training of SPPD. It is observed that both excessively large and small values of \(\gamma\) are detrimental to the training of dynamic value margins in SPPD, thereby affecting the generalization to some extent. However, overall, the performance remains relatively stable, particularly on the GSM8k dataset. This suggests that a balanced choice of \(\gamma\) is crucial for optimizing the effectiveness of the SPPD approach across different models.


\begin{table}[!ht]
\resizebox{0.5\textwidth}{!}{
\begin{tabular}{cccc}

\toprule
Model               & GaoKao2023                                            & OCW                                                   & OlympaidBench*   \\\midrule
Qwen2.5-7B-Base     & 48.0                                                  & 6.3                                                   & 20.5                                                  \\\hline
+SFT-PRM            & 52.2                                                  & 19.1                                                  & 22.8                                                  \\
+SFT-PRM \& DPO-PRM & 55.0                                                  & 16.1                                                  & 23.7                                                  \\
+SPPD              & \begin{tabular}[c]{@{}c@{}}\textbf{56.8}\\ \textcolor{blue}{+1.8\%}\end{tabular} & \begin{tabular}[c]{@{}c@{}}\textbf{20.0}\\ \textcolor{blue}{+4.8\%}\end{tabular} & \begin{tabular}[c]{@{}c@{}}\textbf{26.1}\\ \textcolor{blue}{+2.4\%}\end{tabular} \\
+SPPD+MAJ@64       & 62.6                                                  & 29.4                                                  & 43.3                                                  \\
+SPPD+ORM\_MAX@64  & 63.4                                                  & 28.3                                                  & 41.4                                                  \\
+SPPD+ORM\_VOTE@64 & \begin{tabular}[c]{@{}c@{}}\textbf{64.4}\\ \textcolor{blue}{+9.4\%}\end{tabular}                       & \begin{tabular}[c]{@{}c@{}}\textbf{30.9}\\ \textcolor{blue}{+14.8\%}\end{tabular}                                        & \begin{tabular}[c]{@{}c@{}}\textbf{45.4}\\ \textcolor{blue}{+21.7\%}\end{tabular}    \\ \bottomrule                                   
\end{tabular}
}
\caption{Result on out-domain test datasets. OlympaidBench* denotes we only use OlympaidBench-OE-TO-Math-COMP test dataset.}
\label{tab:ablation-out-domain}
\end{table}

\begin{table}[htb]
\resizebox{0.5\textwidth}{!}{
\begin{tabular}{c|c|c|cc}
\toprule
Model & Method     & Margin & MATH500 & GSM8K \\ \midrule
\multirow{2}{*}{Qwen2.5-7B}&  SPPD  & Dynamic & \textbf{71.00}  & \textbf{89.80} \\ \
&\multirow{2}{*}{\begin{tabular}[c]{@{}c@{}}Step-dpo-\\ fix-margin \end{tabular}} & 0  & 69.60   & 89.40     \\
& & $\gamma^*$ & 70.10    & 89.49    \\ \hline
\multirow{2}{*}{Llama3.1-8B} &SPPD   & Dynamic & \textbf{51.2}  & \textbf{84.8} \\ 
&\multirow{2}{*}{\begin{tabular}[c]{@{}c@{}}Step-dpo-\\ fix-margin \end{tabular}} & 0   &   48.8 &   83.2   \\ 
& & $\gamma^*$                & 49.2    &  83.5  \\
\bottomrule
\end{tabular}
}
\caption{SPPD vs fixed margin step DPO on Qwen2.5-7B-Base and Llama3.1-8B-Instruct. $\gamma^*$ represents $\gamma(V^*(s^w_{t+1})-V^*(s^l_{t+1}))=\gamma^*$ in Formula \ref{equ:step-dpo-gamma}.} 
\label{tab:ablation-fix-margin}
\end{table}

\section{Conclusion}
In this work, we propose SPPD, a self-training with process  preference learning using dynamic value margin. SPPD utilizes the Bellman optimality equation and the online RL objective modeled with MDP and designs a step-level tree self-sampling scheme without any distillation. Moreover, we propose a SFT and DPO scheme using PRM for rejection sampling, making the training of SPPD smothor and more effective. Finally, we theoretically demonstrate that under specific reward constraints, our method is equivalent to on-policy policy gradient optimization.

\newpage
\section*{Limitations}
Several limitations remain in our current work. Firstly, our work relies on the effectiveness of PRM, and studies have shown that PRM's performance varies across different policy models and task scenarios; some PRMs may fail under specific tasks \citep{qwen-prm}. This work neglects the updates of PRM. As policy is continuously iterated, PRM faces the risk of becoming ineffective. Additionally, both PPO and GRPO are modeled based on bandit, and how to integrate MDP modeling with  on-policy methods remains an important subject for future research. 

\bibliography{custom}

\appendix

\section{Data Example}\label{sec:data-example}

We demonstrate the preference trajectories sampled from the tree using SPPD as shown below. As previously mentioned, the chosen step and the rejected step are selected based on their PRM score difference exceeding 0.5.

\begin{tcolorbox}[nobeforeafter, 
    title=Training data format., 
    colframe=darkgray, 
    colback=white, 
    breakable]
$[$\textbf{Question}$]$\\
Evaluate: $64^2-36^2$.\\
$[$\textbf{Common Prefix}$]$\\
Step 1: Recognize that the given expression is of the form $a^2 - b^2$, which can be factorized using the difference of squares method.\\
Step 2: Apply the difference of squares formula, which is $a^2 - b^2 = (a + b)(a - b)$.\\
Step 3: Identify a and b in the given expression. Here, $a = 64$ and $b = 36$.\\
Step 4: Evaluate the sum and difference of a and b.\\
a + b = 64 + 36\\
a - b = 64 - 36\\
Step 5: Multiply the result from step 4 back together.\\
$[$\textbf{Chosen Step}$]$\\
$(64 + 36)(64 - 36)$ \\
Step 6: Perform the addition and subtraction operations.\\
The sum of 64 and 36 is 100. The difference of 64 and 36 is 28. \\
$[$\textbf{Chosen PRM Score}$]$\\
\textcolor{blue}{0.934}\\
$[$\textbf{Rejected Step}$]$\\
Step 6: $(64 + 36) * (64 - 36) = (100) * (28) = 1400$\\
$[$\textbf{Rejected PRM Score}$]$\\
\textcolor{red}{0.258}\\
\end{tcolorbox}

\section{PRM score distribution}\label{prm-distribution}
We use the Skywork-o1-Open-PRM-Qwen-2.5-7B \citep{skywork-prm}model to score 545,990 reasoning trajectories obtained from Qwen2.5-7B-Base \citep{qwen2.5} through Tree-Based Self-Sampling. The score for the $t$-th step of the $i$-th trajectory is denoted as \( v_{t}^{(i)} \). 

First, we calculate three metrics (ORM score, Mean PRM score, and Minium PRM score) on trajectories that produce correct answers and those that result in incorrect answers. If a metric exceeds 0.5, the PRM considers the sample to be a correct trajectory; otherwise, it is deemed an incorrect trajectory. We then compute the PRM accuracy rates under these three metrics, see Table \ref{tab:prm-accuracy}. The experimental results demonstrate that Skywork-o1-Open-PRM-Qwen-2.5-7B exhibits strong discriminative ability for both correct and incorrect trajectories under sampled trajectories. Specifically, the ORM metric shows superior performance in identifying correct trajectories, achieving over 90\% accuracy. In contrast, the minimum PRM score excels in distinguishing incorrect trajectories, reaching an accuracy of 92.5\%. However,  using the mean PRM score, the discriminative ability for correct trajectories is notably higher than for incorrect trajectories. This is because Skywork-o1-Open-PRM-Qwen-2.5-7B  can effectively identify erroneous steps, resulting in high scores (close to 1) before these steps occur, which renders the mean PRM score ineffective for judging incorrect trajectories. Conversely, the minimum PRM score identifies the lower bound of trajectory scoring, making it the most suitable metric for evaluating incorrect trajectories.

\begin{table}[!ht]
    \centering
    \resizebox{1.0\linewidth}{!}{
    \begin{tabular}{c|c|ccc}\toprule
        Metric & \# &ORM & Mean PRM & Minium PRM \\ \midrule
        Correct & 281,983 & 0.908 & 0.920 &  0.705\\ \hline
        Incorrect & 264,007 & 0.870& 0.696 &  0.925\\ \bottomrule
    \end{tabular}
    }
    \caption{Skywork-o1-Open-PRM-Qwen-2.5-7B accuracy.}
    \label{tab:prm-accuracy}
\end{table}

Meanwhile, we divide each trajectory into five equal segments, calculate the average score for each segment, and plot the score distribution in box plots categorized by correct and wrong trajectories, as shown in the Figure \ref{fig:prm-distribution}. The figure indicates that for correct trajectories, PRM assigns relatively high scores to all steps with smaller variance; for wrong trajectories, the segment scores given by PRM tend to decrease on average as they get closer to the answer, with the variance also decreasing, suggesting that PRM's confidence in the wrong trajectory leading to an incorrect answer increases.
\begin{figure}[!ht]
    \centering
    \includegraphics[width=1\linewidth]{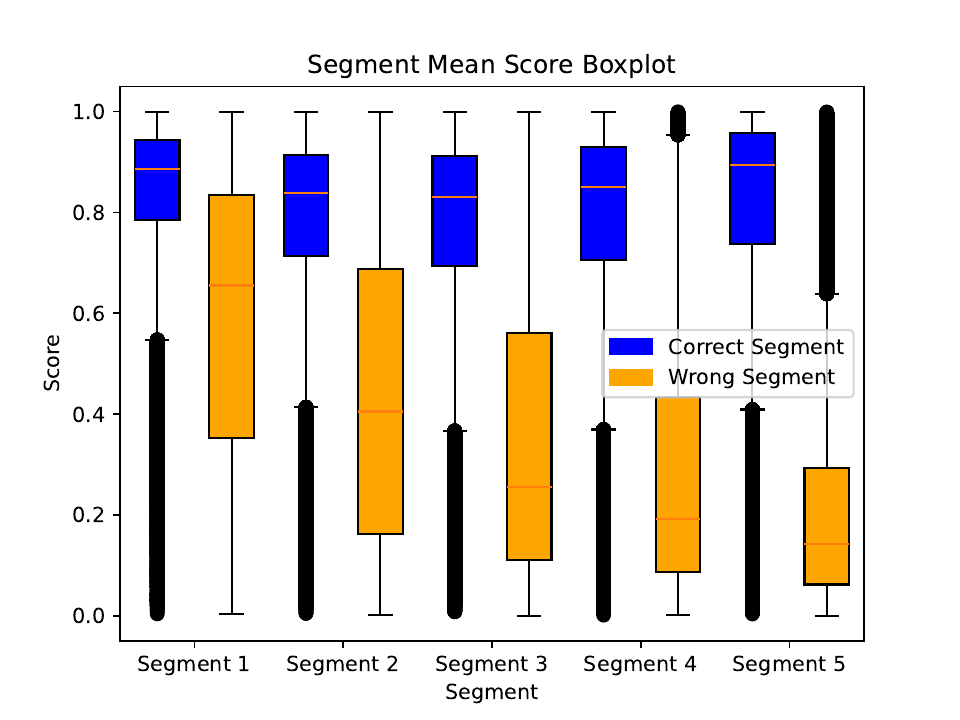}
    \caption{ Skywork-o1-Open-PRM-Qwen-2.5-7B distribution.}
    \label{fig:prm-distribution}
\end{figure}

\section{Evaluation}\label{app:evaluation}
\subsection{Evaluation Prompts}

For a fair evaluation, the same prompt and format is applied to our trained models as well as other open-source models: 
\begin{tcolorbox}[nobeforeafter, 
    title=Prompt used for evaluation., 
    colframe=darkgray, 
    colback=white, 
    breakable]
$[$SYSTEM$]$\\
Please reason step by step and put your answer in $\backslash\backslash$boxed$\{\}$. \\
$[$Question$]$\\
\{question\}.\\
\end{tcolorbox}

\section{Proofs}
\subsection{Proof of Lemma (\ref{lemma:reward-function})}\label{app:prove_for_reward_function}
\begin{lemma}[Optimal Step Reward Function]
    Under the step MDP definition\ref{sec:step-dpo-mdp} and fix solution for the maximum casual entropy problem (Equation (\ref{equ:fix_solution})), the optimal step reward function can be calculate as follow:
    \begin{align}
    r(s_t, a_t) = \underbrace{\beta \log \frac{\pi^*(a_t | s_t)}{\pi_{\text{ref}}(a_t | s_t)}}_{\text{Implicit Reward}} + \underbrace{V^*(s_{t+1}) - V^*(s_t)}_{\text{Value Gain}}.
\end{align}
\end{lemma}

\begin{proof}
    According to the Bellman optimality equation \citep{bellman-equation} in step MDP, we have:
    \begin{align}
    Q^*(s_t, a_t) = r(s_t, a_t) + V^*(f(s_t, a_t)).
    \label{equ:optimal-bellman}
    \end{align}
    Here, if \(s_{t+1} = f(s_t, a_t)\) is a terminal state, then \(V^*(f(s_t, a_t)) = 0\). Meanwhile, if we log-linearize the Equation (\ref{equ:fix_solution}), we have:
    \begin{align}
        Q^*(s_t, a_t) = \beta \log \frac{\pi^*(a_t | s_t)}{\pi_{\text{ref}}(a_t | s_t)} + V^*(s_t).
        \label{equ: other-fix-solution}
    \end{align}
    Therefore, combine the Equation (\ref{equ:optimal-bellman}) \& (\ref{equ: other-fix-solution}), we have:
    \begin{align*}
    r(s_t, a_t) = \underbrace{\beta \log \frac{\pi^*(a_t | s_t)}{\pi_{\text{ref}}(a_t | s_t)}}_{\text{Implicit Reward}} + \underbrace{V^*(s_{t+1}) - V^*(s_t)}_{\text{Value Gain}}.
    \end{align*}
\end{proof}
\subsection{Proof of Theorem \ref{theorem:step-dpo-loss}}\label{app:prove_steo_dpo_loss}
\begin{theorem}[Step DPO  Loss Using  Dynamic Value Margin.]\label{theorem:step-dpo-loss}
   If we aim to minimize the Kullback-Leibler(KL) divergence between the step-level preference distribution \(p_{\text{data}}\) in \(\mathcal{D}_{\text{step}}\) and the model's current preference distribution \(p_{\theta}\) under the sampling of \(\pi_{ref}\), we can obtain the following loss function:
\begin{align*}
    \mathcal{L}_{\text{step-dpo}} &=  -\mathbb{E}_{a^w_{t+1}, a^l_{t+1} \sim \pi_{\text{ref}}(\cdot | s_t)} [ \nonumber \\
   & \log \sigma (\beta h_{\theta}(a^w_{t+1}, a^l_{t+1}) \nonumber \\
    & - (V^*(s^w_{t+1}) - V^*(s^l_{t+1}))) ],
\end{align*}
where $h_{\theta}(a^w_{t+1}, a^l_{t+1}) =\log \frac{\pi_{\theta}(a_t^w | s_t)}{\pi_{\text{ref}}(a_t^w | s_t)} - \log \frac{\pi_{\theta}(a_t^l | s_t)}{\pi_{\text{ref}}(a_t^l | s_t)}$. 
\end{theorem}
\begin{proof}
    According to the Equation (\ref{equ:reward}), we have:
    \begin{align*}
        & p_{\theta}(a^w_{t+1}\succ a^l_{t+1}| s_t) \\
        &= \sigma (\beta h_{\theta}(a^w_{t+1}, a^l_{t+1}) \nonumber  - (V^*(s^w_{t+1}) - V^*(s^l_{t+1})))
    \end{align*}
    So the KL divergence between $p_{\theta}$ and $p_{data}$ under the sampling of \(\pi_{ref}\) is:
    \begin{align*}
        &\mathbb{E}_{a^w_{t+1},a^l_{t+1}\sim \pi_{ref}(\cdot|s_t)}[D_{KL}(p_{data}||p_{\theta}) ]\\
        &= \mathbb{E}_{a^w_{t+1},a^l_{t+1}\sim \pi_{ref}(\cdot|s_t)}[\\&
        p_{data}(a^w_{t+1}\succ a^l_{t+1}| s_t)\log\frac{p_{data}(a^w_{t+1}\succ a^l_{t+1}| s_t)}{p_{\theta}(a^w_{t+1}\succ a^l_{t+1}| s_t)}\\&
       +  p_{data}(a^l_{t+1}\succ a^w_{t+1}| s_t)\log\frac{p_{data}(a^l_{t+1}\succ a^w_{t+1}| s_t)}{p_{\theta}(a^l_{t+1}\succ a^w_{t+1}| s_t)}] \\
   & =-\mathbb{E}_{a^w_{t+1}, a^l_{t+1} \sim \pi_{\text{ref}}(\cdot | s_t)} [\log p_{\theta}(a^w_{t+1}\succ a^l_{t+1}| s_t)],
    \end{align*} which is the same as Equation (\ref{equ:step-dpo-loss}).

\end{proof}

\subsection{Proof of Theorem \ref{theorem:equivalence}} \label{app:theorem:equivalence}
\begin{theorem}[Equivalence Between Offline Step DPO and Online Policy Gradient]
 If we define the reward in Equation (\ref{equ:online-gd}) as \(r(\tau) = \prod_{i=1}^T \frac{\pi_{\text{ref}}(a_t | s_t)}{\pi^p_\theta(a_t | s_t)}\), and define the \textbf{Offline every-step preference loss} as:
\begin{align*}
    \mathcal{L}_{\text{every-step}} = \\ 
    &\mathbb{E}_{\tau \sim \pi^p_{\text{ref}}} \left[- \sum_{t=0}^{T-1}  \log \pi^p_\theta(a_{t+1}^w | s_t) \right],\nonumber
\end{align*}
then the following equivalence holds:
\begin{align*}
    \nabla_{\theta} J(\theta) = -\nabla_{\theta} \mathcal{L}_{\text{every-step}}.
\end{align*}
\end{theorem}
\begin{proof}
    \begin{align*}
&\nabla_{\theta}\mathcal{L}_{every-step} \\
&=\mathbb{E}_{\tau\sim \pi^p_{ref}}[-\sum_{t=0}^{T-1}\nabla_{\theta}\log\pi^p_{\theta}(a_{t+1}^w|s_t))]\\
&= \mathbb{E}_{\tau\sim \pi^p_{\theta}}[-\frac{\pi_{ref}^p(\tau)}{\pi_{\theta}^p(\tau)}\sum_{t=0}^{T-1}\nabla_{\theta}\log\pi^p_{\theta}(a_{t+1}^w|s_t))]\\
&= \mathbb{E}_{\tau\sim \pi^p_{\theta}}[\\ &
-\prod_{i=0}^{T-1}\frac{\pi_{ref}^p(a_{t+1}|s_t)}{\pi_{\theta}^p(a_{t+1}|s_t)} \sum_{t=0}^{T-1}\nabla_{\theta}\log\pi^p_{\theta}(a_{t+1}^w|s_t))] \\
&=\mathbb{E}_{\tau\sim \pi^p_{\theta}}[-r(\tau)\sum_{t=0}^{T-1}\nabla_{\theta}\log\pi^p_{\theta}(a_{t+1}^w|s_t))] \\
&=-\nabla_{\theta} J(\theta).
\end{align*}

\end{proof}
\end{document}